\documentclass[runningheads]{llncs}

\usepackage[utf8]{inputenc}
\usepackage{amssymb}
\usepackage{amsmath}
\usepackage{amsthm}
\usepackage{mathtools}
\usepackage[linesnumbered,ruled,vlined]{algorithm2e}
\usepackage{tikz}
\usepackage{bbm}
\usepackage{thmtools}
\usepackage{thm-restate}
\usetikzlibrary{trees}
\usepackage{biblatex}
\title{Quantum-Inspired Classical Algorithm for Principal Component Regression}
\author{Daniel Chen \inst{1} \and
Yekun Xu\inst{2} \and
Betis Baheri \inst{3} \and 
Chuan Bi \inst{4} \and 
Ying Mao \inst{5} \and 
Qiang Guan \inst{3} \and 
Shuai Xu \inst{1} } 
\authorrunning{D. Chen et al.}

\institute{Case Western Reserve University, Cleveland OH 44106, USA \\
\email{\{txc461, sxx214\}@case.edu} \and
Florida International University, Maimi FL 33199, USA \\
\email{yxu040@fiu.edu} \and
Kent State University, Kent OH 44240, USA \\
\email{\{bbaheri, qguan\}@kent.edu} \and 
National Institute on Aging, National Institute of Health, Baltimore MD 21224, USA \\
\email{chuan.bi@nih.gov} \and
Fordham University, The Bronx NY 10458 \\
\email{ymao41@fordham.edu}}

\begin{document}
\date{}
\maketitle

\begin{abstract}
This paper presents a sublinear classical algorithm for principal component regression. The algorithm uses quantum-inspired linear algebra, an idea developed by Tang \cite{tang2019quantum}. Using this technique, her algorithm for recommendation systems achieved runtime only polynomially slower than its quantum counterpart \cite{kerenidis2016quantum}. Her work was quickly adapted to solve many other problems in sublinear time complexity \cite{chia1910sampling}. In this work, we developed an algorithm for principal component regression that runs in time polylogarithmic to the number of data points, an exponential speed up over the state-of-the-art algorithm, under the mild assumption that the input is given in some data structure that supports a norm-based sampling procedure. This exponential speed up allows for potential applications in much larger data sets.
\keywords{Principal Component Regression \and Quantum Machine Learning  \and Sampling}
\end{abstract}

\section{Introduction}
Principal component analysis (PCA) is a technique of dimensionality reduction developed by Pearson in 1901 \cite{pearson1901liii}. The goal is to reduce the dimension of a data set, $X \in \mathbb{R}^{n\times d}$, into the directions of the largest variance. It turns out that the direction in which the variance is maximized is given by the eigenvector of the covariance matrix corresponding to the largest eigenvalue, the second-most significant vector by the eigenvector corresponding to the second largest eigenvalue, and so on \cite{10.5555/1734076}. If the data set is shifted so that it has mean $0$ for each variable, then PCA is  reduced to finding the spectral decomposition of $X^T X$. This problem can be solved in running time $O(\min(n^2d, nd^2))$ using singular value decomposition (SVD). However, the complexity can be reduced to $O(k^3)$ to obtain just the first $k$ principal components, which is often the concern \cite{zou2006sparse}.

Ordinary least squares regression is another commonly used technique in machine learning. The method finds a set of linear coefficients that best fits some data in the least-squares sense. More formally, given a matrix $X$ and vector $y$, we desire the regression coefficient $\beta$ that minimizes the error, namely, let $\varphi \coloneqq y - X\beta$ and we want to minimize $\|\varphi \|^2$. The $\beta$ that minimizes the least square error in the general case is known to be $(X^T X)^{-1} X^T y$, which corresponds to finding the pseudo-inverse of $X$. This can also be done through SVD with the same complexity described above.  

However, one of the many problems one might encounter in finding the regression is the problem of collinearity. This is when a linear function of a few independent variables is equal to, or close to 0. The consequence for collinearity is that the resulting coefficient ${\beta}$ may be very sensitive to small perturbations in $y$ and the addition or removal of independent variables \cite{rawlings2001applied}. The principal component regression (PCR) is a method for addressing this problem. PCR attempts to keep the variables in which the most variation in the data is account for and drop those that has less predictive power. A standard procedure for finding the PCR is presented in Algorithm 1 \cite{rawlings2001applied}. 

\begin{algorithm}
\caption{Algorithm for PCR}
\KwIn{$X^{n\times d}$, $y^{n \times 1}$, $k$ for the number of principal components wanted}
\KwOut{$\hat{\beta}$}
Perform PCA on $X$ and store the top $k$ principal components into matrix $V ^{d \times k}$ \\
Let $W^{n \times k} = XV = \begin{pmatrix} Xv_1 & Xv_2 & \dots & Xv_k \end{pmatrix}$ \\
Compute $\hat{\gamma}^{k \times 1} = (W^T W)^{-1} W^T Y$, the estimated regression coefficients using ordinary least squares \\
Output $\hat{\beta}^{d \times 1} = V \hat{\gamma}$, the final PCR estimator of $\beta$
\end{algorithm}

PCR was first developed by Maurice G. Kendall, in his book ``A Course in Multivariate Analysis'' in 1957 \cite{kendall_1957}. Since then, the method becomes commonly used for analyzing data sets,  especially those with large number of explanatory variables. In addition to avoiding collinearity problems, it is also robust to noise and missing values. The robustness allows for applications that are more prone to measurement errors like time series forecasting for temporal data \cite{NIPS2019_9181}. PCR has been shown to be useful is a variety of applications across different fields, including genomic analysis \cite{du2018genomic,pant2010principal}, fault diagnosis in process monitoring \cite{peng2015quality} and social network link prediction \cite{bao2013sonlp}.

\subsection{Related Work}

Many work has also been put in to alternating and improving upon the standard PCR procedure to further expand the possible applications. Functional versions of PCR has been developed to solve signal regression problems \cite{reiss2007functional}. Probabilistic PCR as well as mixture probabilistic PCR were developed to more accurately model processes with certain probabilistic behaviors \cite{ge2011mixture}. Computationally efficient PCR has also been developed. For instance, Allen-Zhu and Li developed a fast algorithm for PCR focusing on reducing the time complexity with respect to the number of principal components desired \cite{allen2017faster}. Overall, PCR is a popular method in data analytics with great interest in improving the performance to further extend its application. 

The time complexity of performing PCR using standard methods would run in linear of the number of data points. However, as the demand for larger data set increases, PCR would benefit from an improvement in computational complexity. In this paper, we utilized quantum-inspired methods \cite{chia1910sampling,tang2018quantum} to achieve an exponential speed up to the standard PCR algorithm.

\subsection{Our Methods and Result}

Our work was directly inspired by Tang's algorithm for PCA \cite{tang2018quantum} and Chia, Lin, and Wang's algorithm for solving a linear system \cite{chia2018quantum}. Tang first developed the quantum-inspired methods for fast linear algebra in her work on developing a sublinear algorithm for recommendation systems \cite{tang2019quantum}. Her work was inspired by the quantum algorithm developed by Kerenidis and Prakash \cite{kerenidis2016quantum} and the fast singular value decomposition by Frieze, Kannan, and Vempala \cite{frieze2004fast}. She took advantage of the fact that the quantum algorithms do not output the whole description of the desired vector $x$, but a sample of it with probability $x_i^2 / \|x\|^2$ instead. In the end, she was able to achieve a runtime comparable to that of the quantum counterpart, showing that in some cases, classical algorithms can achieve runtime complexity similar to a quantum algorithm.

With the various sublinear matrix operations available and the development of fast matrix multiplication \cite{chia1910sampling}, the problem of principal component regression became relatively straight-forward to solve. We achieved a runtime complexity poly-logarithmic to the number of data by exploiting the sampling scheme: instead of computing the whole matrix/vector in Algorithm 1, we gain the ability to access it, more importantly, the ability to generate samples (we call this \textit{sample access}, which will be elaborated further later). The outline of the proposed quantum-inspired algorithm follows the following step: 
\begin{enumerate}
    \item Get sample access to $V$, the matrix made up of the top $k$ principal vectors
    \item Get sample access to $W = XV$
    \item Find the pseudo-inverse of $W$, $W^+$
    \item Compute $\hat \gamma = W^+ y$
    \item Output either a sample of $\hat \beta$ or compute an entry of $\hat \beta$ where $\hat \beta = X V$
\end{enumerate}
Each step presented above can be done in sublinear time with respect to the size of the data set with constant success probability. Although due to the probabilistic nature of the procedure, there is a certain amount of error present in each step, we can bound the total error by carefully choosing each parameters. Thus, we present our main theorem as follows: 

\begin{restatable}[]{thm}{One}
Let $X \in \mathbb{R}^{n \times d}$, $y \in \mathbb{R}^n$, both stored in the data structure described in section 3.1. Suppose we are interested in the vector $\hat \beta$, the coefficients obtained from the principal component regression of the equation $y = X\hat\beta + \varphi$. There exists an algorithm that approximates the $i$-th entry or outputs a sample of $\hat \beta$ in 
\begin{align}
O(\textnormal{poly}(k, d, \|X\|_F, \|X\|, \|y\|, \eta,  \frac{1}{\epsilon}, \frac{1}{\sigma'}, \frac{1}{\xi}, \frac{1}{\theta}, \frac{1}{\delta_3}, \frac{1}{\delta_4},)~\textnormal{polylog}(n, d, k, \frac{1}{\delta_6}))
\end{align}
time with a additive error $\epsilon$, constant success probability and $\Tilde{O}(dk)$ extra space. 
\end{restatable}

\section{Preliminaries}

For a vector $x \in \mathbb{R}^n$, $x(i)$ denotes the $i$-th entry of the vector. For an matrix $A \in \mathbb{R}^{n \times d}$, $A(\cdot,i)$ denotes the $i$th column and $A(j, \cdot)$ denotes the $j$th row. The $(i,j)$-th entry of $A$ would be denoted by $M(i,j)$. $\|x\|$ denotes the $\ell_2$ norm for vectors. The spectral norm of $A$ is denoted by $\| A \|$, whereas the Frobenius norm is denoted by $\|A\|_F$. For every matrix $A \in \mathbb{R}^{n\times d}$, there exists a decomposition $A = U\Sigma V^T = \sum_i \sigma_i u_i v_i^T$ where $U, V$ are unitary with column vectors $u_i, v_i$ respectively and $\Sigma$ is diagonal with the $(i,i)$-th entry as $\sigma_i$. We call $u_i$ the left singular vectors, $v_i$ the right singular vectors, and $\sigma_i$ the singular values of $A$. While the inverse of $A$ is denoted by $A^{-1}$, the Moore-Penrose psuedoinverse is given by $A^+$, i.e., if $A$ has the singular value decomposition of $U\Sigma V^T$, $A^+ = V\Sigma^{-1}U^T$. 

The notion of \textit{approximate isometry} is important for algorithms introduced later in the paper. It is defined as follow:
\begin{definition}[Definition 2.1 in \cite{chia1910sampling}]
Let $n, d \in \mathbb{N}$ and $n \geq d$. A matrix $V \in \mathbb{R}^{n \times d}$ is an $\alpha$-approximate isometry if $\|V^TV  - I\| \leq \alpha$.
\end{definition}

Let function $\ell_A(\sigma)$ be the index of the smallest singular value in a matrix, $\{\sigma_1, \dots, \sigma_r\}$ that's greater than the threshold $\sigma$: $\ell_A(\sigma) = \max \{i~|~\sigma_i > \sigma\}$ \cite{chia1910sampling}.
Let $A$ be an arbitrary matrix. Matrix $A_{\sigma,\eta}$ would be a matrix with a transformation on the singular values of $A$, $A = \sum_i f(\sigma_i) u_i v_i^T$ for some function $f$, according to the following rules:
\begin{enumerate}
    \item if the singular values of $A$ is greater or equal to $\sigma(1 + \eta)$, then it remains the same
    \item if less than or equal to $\sigma(1 - \eta)$, then it equals $0$
    \item if in between the two values, it may take on values anywhere from $0$ to its current value
\end{enumerate}
One can think of this as a relaxed notion of a low-rank approximation of $A$, where instead of having a strict cut-off, we have some room for error in between.

For a non-zero vector $x \in \mathbb{R}^n$, to sample from $x$ means to draw an index $i \in [1,n]$ following the distribution 
\begin{align}
    \mathcal{D}_x (i) = \frac{x_i^2}{\|x\|^2}
\end{align}
For a non-zero matrix $A \in \mathbb{R}^{n\times d}$, we define the $n$-dimensional vector $\Tilde{A} = \begin{pmatrix} \|A(1,\cdot)\|^2 & \|A(2, \cdot)\|^2 & \dots & \|A(n,\cdot)\|^2 \end{pmatrix}^T$.

\section{Linear Algebra via Sampling}
\subsection{Data Structure}

We assume the existence of a special data structure to support sampling. This data structure assumption has been used in various other places \cite{chia2018quantum, kerenidis2016quantum, tang2019quantum, tang2018quantum}. The properties of the data structure is described as follow: 

\begin{lemma}
There exists a data structure storing a non-zero vector $x \in \mathbb{R}^{n}$ with $O(n \log n)$ space that does:
\begin{enumerate}
    \item $O(\log n)$ time entry-wise query and updating
    \item $O(1)$ time output of the $\ell_2$ norm of $v$
    \item $O(\log n)$ time sampling from $\mathcal{D}_v$
\end{enumerate}
\end{lemma}

\begin{lemma}
Let $A \in \mathbb{R}^{n \times d}$ be a non-zero matrix, $\Tilde{A} \in \mathbb{R}^n$ where the $i$-th entry stores $\|A(i,\cdot)\|^2$. Then, there exists a data structure storing $A$ in $O(nd \log nd)$ space, supporting
\begin{enumerate}
    \item $O(\log nd)$ time entry-wise query and updating
    \item $O(1)$ time output of $\|A\|_F^2$
    \item $O(\log n)$ time output of $\|A(i,\cdot)\|^2$
    \item $O(\log nd)$ time sampling from $\mathcal{D}_{\tilde{A}}$ or $\mathcal{D}_{A(i,\cdot)}$
\end{enumerate}
\end{lemma}

\begin{figure}
\centering
\begin{tikzpicture}[level distance=1.5cm,
  level 1/.style={sibling distance= 6cm},
  level 2/.style={sibling distance= 3cm}, 
  level 3/.style={sibling distance= 1.5cm}]
  \node {$\|A\|_F^2$}
    child {node {$\|A(1, \cdot)\|^2$}
        child {node {$A(1,1)^2 + A(1,2)^2$}
            child {node {$A(1,1)^2$}
                child {node {$s(A(1,1))$}}
            }
            child {node {$A(1,2)^2$}
                child {node {$s(A(1,2))$}}
            }
        }
        child {node {$A(1,3)^2 + A(1,4)^2$}
            child {node {$A(1,3)^2$}
            child {node {$s(A(1,3))$}}
            }
            child {node {$A(1,4)^2$}
              child {node {$s(A(1,4))$}}
            }
        }
    }
    child {node {$\|A(2, \cdot)\|^2$}
        child {node {$A(2,1)^2 + A(2,2)^2$}
            child {node {$A(1,1)^2$}
                child {node {$s(A(2,1))$}}
            }
            child {node {$A(1,2)^2$}
                child {node {$s(A(2,2))$}}
            }
        }
        child {node {$A(2,3)^2 + A(2,4)^2$}
            child {node {$A(2,3)^2$}
            child {node {$s(A(2,3))$}}
            }
            child {node {$A(2,4)^2$}
              child {node {$s(A(2,4))$}}
            }
        }
    }
    ;
\end{tikzpicture} 

\caption{An example of the data structure for storing a matrix $A \in \mathbb{R}^{2 \times 4}$, represented as a binary tree. We use $s(x)$ to denote the sign function, returning $1$ if $x >0$, $-1$ if $x < 0$, $0$ otherwise.}
\end{figure}
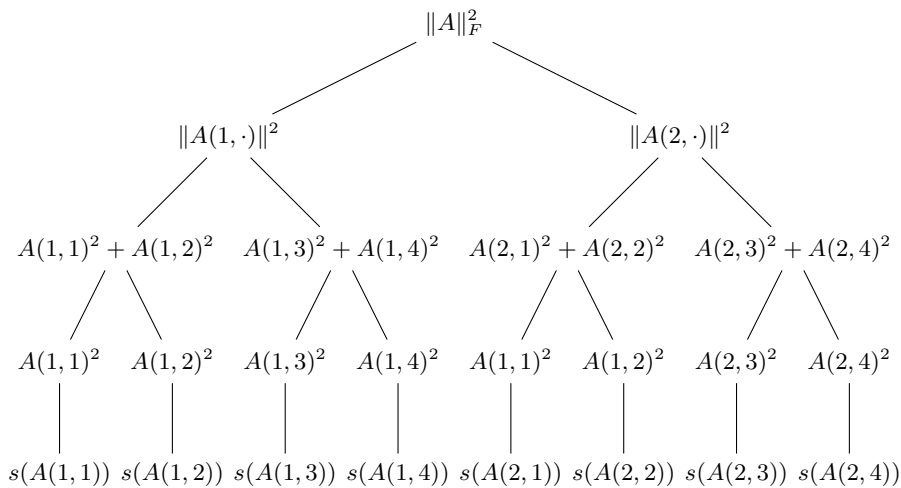

Figure 1 shows one way of constructing the data structure through binary trees \cite{tang2019quantum}. The example is done in a $2 \times 4$ matrix, $A$. Storing a vector would be similar, where one can see as a subtree starting from the second level. 

\subsection{Technical Lemmas}

In this section, we introduce the important technical tools utilizing the sampling scheme and data structure described above. These results are developed and proven mostly by Tang \cite{tang2019quantum}. However, Chia \textit{et al.} complied and generalized many of the procedures. First, we introduce the general notation for sampling and query access  \cite{chia1910sampling}:

\begin{definition}
For a vector $v \in \mathbb{R}^n$, we have $Q(v)$, \textit{query access} to $v$ if for all $i \in [n]$, we can obtain $v(i)$. Similarly, we have query access to matrix $A \in \mathbb{R}^{n\times d}$ if for all $(i,j) \in [n]\times [d]$, we can obtain $A(i,j)$.
\end{definition}

\begin{definition}
For a vector $v \in \mathbb{R}^n$, we have $SQ_\nu (v)$, sample and query access to $v$ if we: 
\begin{enumerate}
    \item can perform independent samples from $v$ following the distribution $\mathcal{D}_v$ with expected cost $s(v)$
    \item have query access $Q(v)$ with expected cost $q(v)$
    \item obtain $\|v\|$ to multiplicative error $\nu$ with success probability $9/10$ with expected cost $n_\nu (v)$
\end{enumerate}
\end{definition}

\begin{definition}
For a matrix $A \in \mathbb{R}^{n \times d}$, we have $SQ_{\nu_1}^{\nu_2}$, sampling and query access to to $A$ if we:
\begin{enumerate}
    \item have $SQ_{\nu_1}(A(i,\cdot))$, for all $i \in [m]$, as described in Definition 2.2 with cost $s(A)$, $q(A)$, $n_{\nu_2}(A)$ respectively
    \item can sample from $\Tilde{A}$ following the distribution $\mathcal{D}_{\Tilde{A}}$
    \item can estimate $\|A\|_F^2$ to multiplicative error $\nu_2$ in cost $n^{\nu_2}(v)$
\end{enumerate}
\end{definition}

So, for example, if $n$-dimensional vector is stored in the data structure described in Lemma 1, $s(v) = q(v) = O(\log n)$. Furthermore, for some vector $v$, we denote $sq_\nu(v) = s(v) + q(v) + n_\nu(v)$. Similarly, we denote $sq_{\nu_1}^{\nu_2}(A) = s(A) + q(A) + n_{\nu_1}(A) + n^{\nu_2}(A)$. 

Based off of these notions of sample and query access, we can implement approximations of many linear algebraic operations. We introduce the algorithm for, in order, estimating inner product, generating samples for matrix-vector multiplication, and estimating matrix multiplications \cite{chia1910sampling}. 

\begin{lemma}
For some $x,y \in \mathbb{R}^n$, given $SQ(x)$, $Q(y)$, we can estimate  $\langle x, y \rangle$ to additive error $\epsilon$ and failure probability $\delta$ in query and time complexity \\
$O(\|x\|^2\|y\|^2 \frac{1}{\epsilon} \log\frac{1}{\delta} (s(x) + q(x) + q(y))+n(x))$ 
\end{lemma}

\begin{lemma}
Given $SQ(V^T) \in \mathbb{R}^{k\times n}$ and $Q(w) \in \mathbb{R}^k$, we can get $SQ_\nu(Vw)$ with (expected) time complexities $q(Vw) = O(kq(V)+q(w))$, $s(Vw) = O(k \mathcal{C}(V,w)(s(V) + kq(V)+ kq(w)))$ where \begin{align}
    \mathcal{C}(V,w) = \frac{\sum_i w(i)^2 \|V(\cdot,i)\|^2}{\|\sum_i w(i)V(\cdot,i)\|^2}
\end{align}
\end{lemma}

\begin{lemma}
Let $A \in \mathbb{R}^{n \times d}$, $B \in \mathbb{R}^{n \times p}$, and $\epsilon, \delta \in (0,1)$ be error parameters. We can get a concise description of $U \in \mathbb{R}^{m \times r}, D \in \mathbb{R}^{r \times r}, V \in \mathbb{R}^{n \times p}$ such that with 9/10 probability $\|AB - UDV\|_F \leq \epsilon$. $r \leq O(\|A\|_F^2\|B\|_F^2/\epsilon^2)$. $\|U^TU - I\|\leq \delta$, and $\|VV^T - I\|\leq \delta$. And we can compute a succinct description in time $O(t^3 + \frac{t^2}{\delta^2}(sq(A^T)+sq(B)))$, $q(U) = O(t q(A^T))$, $q(V) = O(t q(B))$ where $t = O(\|A\|_F^2 \|B\|_F^2 / \epsilon^2)$.
\end{lemma}

\begin{corollary}[proof of Theorem 3.1 in \cite{chia1910sampling}]
Given $SQ(V)$, $SQ(U),$ and a diagonal matrix $D^{r\times r}$ such that $\|X - UDV\|_F \leq \epsilon$, then one can obtain $SQ(UDV)$ where $sq(UDV) = \Tilde{O}(r^2 \max(sq(U),sq(V))).$
\end{corollary}
This corollary was used in quantum-inspired support vector machine to gain sample access to a matrix product \cite{chia1910sampling}.

The following lemma acquiring the approximate low-rank approximation \cite{chia1910sampling}, which is useful for both PCA and finding the pseduo-inverse.

\begin{lemma}
Let $A = A^{(1)} + A^{(2)} + \dots + A^{\tau}$. Given $SQ(A^{(l)})$, $SQ(A^{(l)^T})$, a singular value threshold $\sigma > 0$, and error parameters $\epsilon, \eta \in (0,1)$, there exists an algorithm which gives $D$ and a succinct description of $\check{U}$ and $\check{V}$ with probability 9/10 in time complexity $O(\frac{\tau^{18}(\sum_l \|A\|^{(l)^2}_F)^{12}}{\epsilon^{12} \sigma^{24} \eta^6}sq(A^{[\tau]}))$, where $\Check{U}, \Check{V}, D$ satisfies that
\begin{enumerate}
    \item $\check{U} \in \mathbb{R}^{m \times r}, \check{V}^{n \times r}$ are $O(\eta \epsilon^2 / \tau)$-approximate isometries, $D$ is a diagonal matrix, $\ell_A(\sigma(1 + \eta)) \leq r \leq \ell_A(\sigma(1 - \eta)) = O(\frac{\|A\|_F^2}{\sigma^2 (1 -\eta)^2}) $
    \item $\|A_{\sigma,\eta} - \check{U}D\check{V}^T \|_F \leq \epsilon\sqrt{\sum_l \|A^{(l)}\|_F^2}/\sqrt{\eta}$
\end{enumerate}
We also obtain $SQ_{\eta\epsilon^2/\tau}^{\eta\epsilon^2/\tau}(V)$ where $s(\check{V}) = O(\frac{\tau^{13}(\sum_l \|A^{(l)}\|^2_F))^9}{\epsilon^8 \sigma^{18}\eta^4}q(A^{[\tau]}))$ and $q(\check V) = O(\frac{\tau^7 (\sum_l \|A^{(l)}\|_F^2)^4}{\epsilon^4 \sigma^8 \eta^2}q(A^{[\tau]}))$.
\end{lemma}

\begin{corollary}[Theorem 9 in \cite{tang2018quantum}]
Given $SQ(A) \in \mathbb{R}^{n \times d}, \sigma, k, \eta$ with the guarantee that for $i \in [k]$, $\sigma_i \geq \sigma$ and $\sigma_i^2 - \sigma_{i+1}^2 \geq \eta\|A\|_F^2$, then for the case of $\tau = 1$, there exists a method such that the matrix $\check V$ the algorithm in Lemma 6 outputs satisfies $\| V - \check V\|_F \leq \sqrt{k}\epsilon_v$ for some $\epsilon_v \in (0, 0.01)$.
\end{corollary}

\begin{proof}
In Theorem 9 of \cite{tang2018quantum}, Tang developed an algorithm, implementing the one described in Lemma 6, for principal component analysis. First, define $\epsilon_\sigma, \epsilon_v, \delta \in (0,0.01)$ to be, respectively, the additive error of the singular value squared, the additive error of each right singular vector, and success probability. By running Lemma 6 using the parameters $\epsilon = \min(\frac{\epsilon_\sigma\|A\|_F^3}{\sigma^3}, \epsilon_v^2 \eta, \frac{\sigma}{4\|A\|_F^2})$ and $\sigma' = \sigma - \epsilon\|A\|_F$, the right error of the right singular vectors are bounded by $\epsilon_v$. 

The matrix $\check V$ is constructed by concatenating the $k$ right singular vectors, and each with an additive error of at most $\epsilon_v$. Thus, $\check V$ would have an additive error of at most $\sqrt{k}\epsilon_v$, with respect to the Fronbenius norm.
\end{proof}

\begin{lemma}
Let $A \in \mathbb{R}^{n\times d}$, $\|A\| < 1$. For $\epsilon',\xi,\theta \in (0,1), \theta \leq \sigma_{min}^2,$ \\$ \eta = \Tilde{O}(\epsilon'^2\xi^2\theta^4/\|A\|_F^2)$, where $\sigma_{min}$ is the minimal non-zero singular value of $A$. There exists an algorithm with time complexity $\Tilde{O}((\frac{\|A\|_F^2}{\epsilon'^2 \xi^2\theta^4})^{18}sq(A))$ which provides succinct description to $\title{O}(\epsilon'^2\xi^2\theta^4/\|A\|_F^2)$-approximate isometries $\check{V},\check{U}$ and diagonal matrix $D^{r\times r}$ such that matrix $B = \check{U}D\check{V}^T$ satisfies $\|B - A^+\| \leq \epsilon'$ with probability 9/10. Moreover, $s(B) = O(r^2 \max(sq(\check{U}),sq(\check{V})))$, $q(B) = O(r\max(sq(\check{U}),sq(\check{V})))$, and $n_\eta(B) = O(r^2/\eta^2\max(sq(\check{U}),sq(\check{V})))$ where $r = \Tilde{O}(\frac{\|A\|_F^2}{\epsilon'^2\xi^2\theta^4})$. $sq(\check V, \check U) = \tilde{O}(\frac{\|A\|^{26}}{\epsilon'^{26}\xi^{26}\theta^{52}}sq(A))$
\end{lemma}

Lemma 7 \cite{chia1910sampling} allows us to find the pseudo-inverse of a matrix, which is an important step in finding the regression parameters.

\begin{lemma}
Let $X,Y \in \mathbb{R}^{n\times n}$ be positive semi-definite matrices with \\ $\max \{$rank$(X), $rank$(Y)\} = k$. Let $\sigma_{min} = \min \{\sigma_{min}(X), \sigma_{min}(Y) \} $, where $\sigma_{min}(\cdot)$ is the minimum non-zero singular value of a matrix. Then, 
\begin{align}
    \|X^+ - Y^+\|_F \leq \frac{3\|X - Y\|_F}{\sigma_{min}^2}
\end{align}
\end{lemma}

The above lemma \cite{chia2018quantum} implies that if two matrices are close, then their pseudo-inverses are also close. This lemma would help bounding the error of the algorithm.



\section{Main Algorithm}

We present the main algorithm for performing principal component regression. For this algorithm, we assume the given data is centered, i.e. $y$ and each column of $X$ has $0$ mean. This step is important as PCA is sensitive to the mean of the data set. If the data is not centered, the right singular vectors will not correspond to the eigenvectors of the covariance matrix. Thus, we assume the data has zero mean in each column such that we can directly perform PCA on $X$ through SVD. The algorithm satisfying Theorem 1 is as described in Algorithm 2. 

\begin{algorithm}
\caption{Quantum-Inspired PCR}
\KwIn{$SQ(X)$, $SQ(X^T)$, $SQ(y)$, $k$ for the number of principal components wanted}
\KwOut{query or sample access to $\hat{\beta}^d$}
Gain $Q(\hat{V})$ where $\hat{V}$'s columns are the principal right singular vectors of $X$ \\
Gain $SQ(\hat V)$ and $SQ(\hat V^T)$ by querying $\hat V$ and store it inside the data structure described in section 3.1\\
Gain $SQ(\hat{W}) = SQ(X \hat V)$ through $SQ(X^T)$ and $SQ(\hat{V})$\\
Gain $SQ(\widehat{W^T W})$ \\
Gain $Q((\widehat{{W}^T W)^+})$\\
Construct vector $\hat \omega \in \mathbb{R}^k$ where $\hat \omega(i) = \langle y, \hat W^T(i, \cdot) \rangle$ and store it in the data structure described in section 3.1 \\
Compute $\hat\gamma$ by matrix vector multiplication $(\widehat{{W}^T W)^+} \hat \omega$ \\ 
Output one of the following :
\begin{itemize}
    \item Query access to the $i$-th entry of $\hat{\beta}$ by computing the inner product $\langle \hat{\gamma}, \hat{V}(i, \cdot ) \rangle$
    \item Sample access for $\hat \beta$ through the matrix-vector multiplication, $\hat V \hat \gamma$
\end{itemize}
\end{algorithm}

\One*

\begin{proof}
We denote the parameters used in each step with the symbol corresponding to each lemmas in section 3.2. If one symbol was used more than once then we distinguish them with subscript numbers corresponding to the step numbers used in Algorithm 2. For instance, the parameter $\epsilon$ used in the second step when implementing Lemma 6 will be denoted as $\epsilon_2$. For step 1, we make the distinction between $\epsilon_1$ and $\epsilon_v$. The former is the parameter actually used in running Lemma 6 through Corollary 2.2, while the later is the error for the low-rank approximation that Lemma 6 is initially trying to achieve.

Furthermore, we make the distinction between 3 notations. Suppose we are estimating a matrix $X$, then $X$ will denote the theoretically correct value, $\widetilde X$ will denote the theoretically estimated value we got from previous error prone steps, and $\hat X$ will denote the actual value we get. 

In step 1, algorithm described in Lemma 6 and Corollary 2.2 implies $\|V - \hat V \|_F \leq \sqrt{k}\epsilon_v$ with probability 9/10. Step 2 simply queries the result from 1, which does not cause any additional error. 

In step 3, we consider the difference between $W$ and $\hat W$ by separately considering $\|W - \widetilde W\|_F$ and $\|\widetilde W - \hat W\|_F$. From Lemma 5, we know that $\|\widetilde W - \hat W\|_F \leq \epsilon_3$ and 
\begin{align}
\|W - \widetilde W\|_F = \|XV - X\hat V\|_F \leq \|X\|\sqrt{k}\epsilon_v
\end{align}
So, by triangle inequality, 
\begin{align}
    \|W - \hat W\|_F \leq \|X\|\sqrt{k}\epsilon_v + \epsilon_3
\end{align}

We bound the error of step 4 following similar logic, we can show since
\begin{align}
    \|W^TW -  {\hat{W}^{T}\hat W \|_F} &= \|W^TW - W^T\hat W + W^T \hat W - \hat W^T \hat W \|_F \\ 
    &= \|W^T (W - \hat W) + (W^T - \hat W^T)\hat W\|_F \\ 
    & \leq (2+\eta\epsilon_1^2)\|X\|(\|X\|\sqrt{k}\epsilon_v + \epsilon_3)
\end{align}
Therefore, $\|W^T W - \widehat{W^T W}\|_F \leq (2+\eta\epsilon_1^2)\|X\|(\|X\|\sqrt{k}\epsilon_v + \epsilon_3) + \epsilon_4$, where $\epsilon_4$ is the error parameter of step 4 from applying Lemma 5. 

Again, following similar procedure, and with Lemma 8, we can show that the error resulted from step 5 can be written as
\begin{align}
    \|(W^T W)^+ - \widehat{(W^T W)^+} \|_F \leq \frac{3((2+\eta\epsilon_1^2)\|W\|(\|X\|\sqrt{k}\epsilon_v + \epsilon_3)+\epsilon_4)}{\theta} + \epsilon_5
\end{align}

For step 6, we first know that $\| \omega - \widetilde \omega\| \leq \|Y\|(\|X\|\sqrt{k}\epsilon_v + \epsilon_3)$, using the same technique as in equation 11. Then, we know that for each entry of $\omega$, it can deviate at most $\epsilon_6$ by Lemma 3. Assume for each of the $k$ computations, we use the same parameter $\epsilon_6$. Thus,
\begin{align}
    \|\hat \omega - \widetilde\omega \| \leq \sqrt{\sum_i \epsilon_6^2} \leq \sqrt{k}\epsilon_6
\end{align}
Therefore,
\begin{align}
    \|\omega - \hat\omega\| \leq \|Y\|(\|X\|\sqrt{k}\epsilon_v + \epsilon_3) + \sqrt{k}\epsilon_6
\end{align}

Following technique used in equation 7, and since $\theta$ is the smallest non-zero singular value of $W^T W$, we get the following :
\begin{align}
    \|\gamma - \hat\gamma\| & = \|(W^T W)^+ \omega- \widehat{( W^T W)^+} \hat \omega \| \\
    & \leq \|\hat\omega\|\frac{3}{\theta}((2+\eta\epsilon_1^2)\|X\|(\|X\|\sqrt{k}\epsilon_v + \epsilon_3)+\epsilon_4)+ \epsilon_5 + \nonumber \\
    & \hspace{3.5cm} \|(W^TW)^+\|(\|y\|(\|X\|\sqrt{k}\epsilon_v + \epsilon_3) + \sqrt{k}\epsilon_6) \\
    &\leq \frac{3}{\theta}\|\hat W\|\|Y\|((2+\eta\epsilon_1^2)\|X\|(\|X\|\sqrt{k}\epsilon_v + \epsilon_3)+\epsilon_4) + \epsilon_5 + \nonumber \\ 
    & \hspace{5cm} \frac{1}{\theta} (\|y\|(\|X\|\sqrt{k}\epsilon_v + \epsilon_3) + \sqrt{k}\epsilon_6)
\end{align}

And finally, the difference our output and the theoretical value of the regression coefficient is bounded using the same technique from equation 7. And since
\begin{align} 
\|\gamma\| = \|(W^T W)^+ W^T y \| \leq \|(W^T W)^+\|\|X\|\|V\|\|y\| \leq \|X\|\|y\|/\theta 
\end{align}
we have
\begin{align}
    \|\beta - \hat\beta\| & \leq \|\hat V\|\|\gamma - \hat\gamma\| + \|V - \hat V\| \|\gamma\| \\
    & \leq \frac{1}{\theta} \left(3 (1 + \eta \epsilon_1^2) \|\hat W\|\|y\|((2+\eta\epsilon_1^2)\|X\|(\|X\|\sqrt{k}\epsilon_v + \epsilon_3)+\epsilon_4) \right) \nonumber \\
    &\hspace{1.2cm} + \epsilon_5 + \frac{1}{\theta} \left( \|y\|(\|X\|\sqrt{k}\epsilon_v + \epsilon_3) + \sqrt{k}\epsilon_6 + \|y\|\|X\|\sqrt{k\theta}\epsilon_v \right) \\
    & =  O(\frac{\|y\|\|X\|\sqrt{k}(\sqrt{\theta} +1)}{\theta}\epsilon_v + \frac{\|X\|^3\|y\|\sqrt{k}\epsilon_v}{\theta} + \frac{\|X\|^2\|y\| + \|y\|}{\theta}\epsilon_3 \nonumber \\
    & \hspace{5.5cm} + \frac{\|X\|^2\|y\|}{\theta}\epsilon_4 + \epsilon_5 + \frac{\sqrt{k}\epsilon_6}{\theta})
\end{align}

Then, we have to pick the appropriate parameters such that the overall error is of order $\epsilon$.  Therefore, we get the following rule for choosing the parameters of each step: 
\begin{itemize}
    \item $\epsilon_v = \min (\frac{\theta\epsilon}{\|y\|\|X\|\sqrt{k}(\|X\|^2 + \sqrt{\theta} +1)}, 0.01)$
    \item $\epsilon_1 = \min(\frac{\epsilon_\sigma\|X\|_F^3}{\sigma^3}, \epsilon_v^2 \eta, \frac{\sigma}{4\|X\|_F^2}, 1)$
    \item $\epsilon_3 = O(\frac{\theta \epsilon}{\|X\|^2\|y\| + \|y\|})$
    \item $\epsilon_4 = O(\frac{\theta\epsilon}{\|X\|^2\|y\|})$
    \item $\epsilon_5 = O(\epsilon)$
    \item $\epsilon_6 = O(\frac{\theta\epsilon}{\sqrt{k}})$
\end{itemize}

Then we analyze the runtime complexity of the whole algorithm. For simplicity in notation, we use $RT(i)$ to represent the runtime for step $i$. 

For the first step, the query access is only given onc the succinct description is computed. Thus, 
\begin{align*}
    RT(1) = O(\frac{\|X\|_F^{24}}{\epsilon_1^{12} \sigma'^{24}\eta^6} \log nd)
\end{align*}

Step 2 requires us to build $V^T \in \mathbb{R}^{k \times d}$ into the data structure, which would take $O(q(V)dk\log dk)$ time. From step 1, we get the $q(V) = O(\frac{\|X\|_F^{24}}{\epsilon_1^{12}\sigma^{24}\eta^6}\log nd)$. The order of which to build in the data structure first does not matter, but the complexity of building the first dominates the second. So, the complexity of step 2 can be expressed as
\begin{align}
    RT(2) = O(kd\frac{\|X\|_F^{18}}{\epsilon_1^8 \sigma'^{18} \eta^4}\log nd \log kd)
\end{align}

Step 3 gains the sample and query access to matrix multiplication by Lemma 5. And since the $\ell_2$ norm of each column of $\hat V$ is at most $1 + \epsilon_v$, $\|\hat V\|_F = O(\sqrt{k})$. Thus, the runtime of step 3 is given by
\begin{align}
    RT(3) &= O(\frac{\|X\|_F^6\|\hat V\|_F^6}{\epsilon_3^6} + \frac{\|X\|_F^4 \|\hat V\|_F^4}{\epsilon_3^4 \delta_3^2}(\log nd + \log dk)) \\
    &= O(\frac{k^3\|X\|_F^{6}}{\epsilon_3^6} + \frac{k^2\|X\|_F^4}{\epsilon_3^4 \delta_3^2}(\log nd + \log dk))
\end{align}

Step 4 is similar to step 3, but samples from $\hat W$, the result from step 3. The complexity is given by 
\begin{align}
    RT(4) &= \tilde O(\frac{\|\hat W\|^{12}}{\epsilon_4^6} + \frac{\|\hat W\|_F^8}{\epsilon_4^4\delta_4^2}(\frac{\|X\|_F^2 \|\hat V\|_F^2}{\epsilon_3^2}\log nd)) \\
    &= \tilde O(\frac{\|X\|_F^{12}}{\epsilon_4^6} + \frac{k\|X\|_F^{10}}{\epsilon_3^2 \epsilon_4^4 \delta_4^2}(\log nd))
\end{align}

Step 5 uses Lemma 7 to invert $ \widehat{W^T W}$. The time needed for acquiring the succinct description is
\begin{align}
    RT(5) &= \tilde O(\frac{\|\widehat{W^T W}\|_F^{36}}{\epsilon_5^{36}\xi^{36} \theta^{72}} \frac{\|\hat W\|_F^4}{\epsilon_4^2} \frac{\|X\|_F^2 \|\hat V\|_F^2}{\epsilon_3^2} \log nd) \\
    &= \tilde O(\frac{k\|X\|^{36}\|X\|_F^{42}}{\epsilon_3^2 \epsilon_4^2 \epsilon_5^{36}\xi^{36}\theta^{72}})
\end{align}

Step 6 constructs a $k$-dimensional vector by performing $k$ inner-products using Lemma 3. So the complexity can be expressed as
\begin{align}
    RT(6) &= O(\sum_{i = 1}^k \|y\|^2 \|\hat W^T(i, \cdot) \|^2 \frac{1}{\epsilon_6} \log \frac{1}{\delta_6} q(\hat W^T)) \\
    &= \tilde O(\|y\|^2 \|\hat W^T\|_F^2 \frac{1}{\epsilon_6} \frac{\|X\|_F^2 \|V\|_F^2}{\epsilon_3^2} \log \frac{1}{\delta_6} \log nd ) \\
    &= \tilde O(\frac{k\|y\|^2 \|X\|_F^4}{\epsilon_3^2 \epsilon_6} \log \frac{1}{\delta_6} \log nd)
\end{align}

Step 7 performs a matrix vector multiplication between a $k \times k$ matrix, $\widehat{({W}^T W)^+}$, and a $k$ dimensional vector, $\hat \omega$, and then updating the entries once all done. This takes $O(k^2 q(\widehat{{W}^T W)^+}) + kq(\hat \omega) + k\log k)$ in total. Substituting in the corresponding complexities, we get
\begin{align}
    RT(7) &= O(\frac{k^2\|\widehat{({W}^T W)^+} \|_F^{28}\|W\|_F^4 \|X\|_F^2 \|V\|_F^2}{\epsilon_3^2 \epsilon_4^2 \epsilon_5^{28}\xi^{28}\theta^{52}}\log nd + k\log k + k\log k) \\
     &= O(\frac{k^2\|X\|^{28}\|X\|_F^{34}}{\epsilon_3^2 \epsilon_4^2 \epsilon_5^{28}\xi^{28}\theta^{52}}\log nd)
\end{align}

Step 8 consists of two possibilities. We clarify each outcome by a subscript $q$ for query access and $s$ for sample access. Each entry can be computed by straightforwardly computing the inner product. The runtime for outputting one entry is given by
\begin{align}
    RT(8_q) &= O(k(q(\hat V^T) + q(\hat \gamma))) \\
    &= O(k\log dk) 
\end{align}
Outputting a sample can be done using Lemma 4, which takes 
\begin{align} 
RT(8_s) &= O(kC(\hat V, \hat\gamma)(\log k + k\log dk + k\log k)) \\
&= O(k^2C(\hat V, \hat\gamma)\log dk) \\
&= O(\frac{k^2\|X\|^2\|y\|^2}{\theta^2} \log dk)
\end{align} 
time, since 
\begin{align}
    C(\hat V, \hat\gamma) &= \frac{\sum_{i = 1}^k \hat\gamma(i)^2 \|\hat V(\cdot, i)\|^2}{\| \sum_{i = 1}^k \hat\gamma(i)\hat V(\cdot,i)\|^2} \leq \frac{\|\hat\gamma\|^2 \|\hat V\|_F^2}{\|\hat V\hat\gamma\|^2} \leq \frac{\|\hat\gamma\|^2\|\hat V\|_F^2}{\min_i \hat\gamma(i)^2 \cdot \|\hat V\|_F^2} \\
    &= O(\|\hat\gamma\|^2) = O(\|\widehat{{W}^T W)^+} \hat W^T y\|^2) = O\left ((\frac{\|X\|\|y\|}{\theta})^2\right )
\end{align}

If we substitute the corresponding errors back into the runtime analysis, we can see that the overall runtime is of the form 
\begin{align}
    RT = O(\textnormal{poly}(k, d, \|X\|_F, \|X\|, \|y\|, \frac{1}{\eta},  \frac{1}{\epsilon}, \frac{1}{\sigma'}, \frac{1}{\xi}, \frac{1}{\theta}, \frac{1}{\delta_3}, \frac{1}{\delta_4},)~\textnormal{polylog}(n, d, k, \frac{1}{\delta_6}))
\end{align}

\end{proof}






\section{Conclusion}

Via the quantum-inspired methods, we were able to perform an approximate PCR for which the regression coefficient deviates by some small error specified by the user. The algorithm has a runtime complexity poly-logarithmic to the size of the input data set. This is a significant speed up that potentially allows for applications on exponentially larger data sets upon implementation. The quantum-inspired methods can provide accurate approximate solutions to many other machine learning algorithms to achieve a similar improvement in time complexity, especially those that involves intensive matrix operations. Furthermore, the methods can have significant implications for complexity theorists regarding the complexity class of quantum machine learning algorithms, or quantum algorithms in general. However, Arrazola, Delgado, Bardhan, and Lloyd have indicated possible implementation issues of the quantum-inspired methods \cite{arrazola2019quantum}. Specifically, the quantum-inspired method might not be optimal in dealing with high rank and sparse matrices, the ones typically seen in practical usage. Nevertheless, the quantum-inspired framework is encouraging for the machine learning community as it can potentially offer many more algorithms at least a polynomial speed up, making them more feasible for larger data sets.

\section{Acknowledgements}
We thank Ning Xie at Florida International University, and Bo Fang at Pacific Northwestern National Laboratory for giving helpful comments on writing this paper.

\printbibliography

\end{document}